\documentclass[sigconf,nonacm]{acmart}

\setcopyright{none}
\renewcommand\footnotetextcopyrightpermission[1]{}

\usepackage{algorithm}
\usepackage{algorithmic}
\usepackage{colortbl}
\usepackage{tabularx}
\usepackage{newfloat}
\usepackage{listings}
\usepackage{booktabs}

\newcommand{\Ind}{\mathbb{I}}
\newtheorem{theorem}{Theorem}
\newcommand{\method}{\textsc{PROVE-Rec}}

\newcommand{\sg}{\operatorname{sg}}
\newcommand{\TopK}{\operatorname{TopK}}

\theoremstyle{definition}

\theoremstyle{remark}

\title[Reasoning with Evidence, Not Merely Rationales: Verifiable Preference Proofs for LLM-Based Recommendation]{Reasoning with Evidence, Not Merely Rationales:\\
Verifiable Preference Proofs for LLM-Based Recommendation}

\author{Yu Hou}
\affiliation{%
  \institution{School of Mathematics and Computing (Computational Science and Engineering), Yonsei University}
  \city{Seoul}
  \country{Republic of Korea}}
\email{houyu@yonsei.ac.kr}

\author{Nathaniel Kang}
\affiliation{%
  \institution{School of Computer Science and Engineering, Kyungpook National University}
  \city{Daegu}
  \country{Republic of Korea}}
\email{natekang@knu.ac.kr}

\author{Pengkai Wang}
\affiliation{%
  \institution{Department of Electrical and Computer Engineering, Sungkyunkwan University}
  \city{Suwon}
  \country{Republic of Korea}}
\email{wangpengkai@g.skku.edu}

\author{Hua Li}
\authornote{Corresponding author.}
\affiliation{%
  \institution{Department of Industrial Engineering, Yonsei University}
  \city{Seoul}
  \country{Republic of Korea}}
\email{li_hua611@yonsei.ac.kr}

\begin{document}

\begin{abstract}
Large language models (LLMs) can infer user preferences from interaction histories and reviews, yet the rationales they generate may not reflect the information actually used for recommendation. A preference claim may be weakly supported by its selected evidence, or may have little effect on the final ranking. We refer to these two failures as the \emph{grounding--influence gap}. We introduce \method{}, a general framework for verifiable preference reasoning in LLM-based recommendation. Pass A converts the complete pre-target history into a compact preference proof consisting of positive and avoidance claims linked to selected evidence entries. Pass B predicts the next item using only the proof and its selected evidence, preventing the recommender from bypassing the reasoning path. To verify \emph{evidence-to-proof grounding}, we compare the effect of masking selected evidence with masking a comparable control entry. To verify \emph{proof-to-recommendation influence}, we remove a preference claim and measure the resulting decrease in the target item’s ranking margin. A ranking-preservation objective further retains useful information from the complete history. Comprehensive experiments on wide-ranging real-world datasets demonstrate that \method{} consistently outperforms strong sequential, generative, and LLM-enhanced baselines, with improvements of up to 7.45\%. Controlled ablations confirm the effectiveness of the two-pass architecture and verification objectives. Moreover, \method{} produces claims that are more strongly grounded in historical evidence and more influential to recommendation while preserving ranking quality.
\end{abstract}


\maketitle

\section{Introduction}
\label{sec:introduction}


LLM-based recommendation is moving beyond item-pattern matching toward richer preference modeling and reasoning \citep{geng2022p5,bao2023tallrec,liao2024llara}. Modern systems can combine collaborative histories with language knowledge, item text, and user reviews before ranking or generating items \citep{zheng2024lcrec,wei2024llmrec,kim2024large,liu2025onerecthink,he2026sidreasoner}. Recent studies further explore rationale-enhanced user modeling, unified LLMs for multiple recommendation tasks, and reinforcement-learning-based optimization of recommendation behavior \citep{wang2024rdrec,jiang2025universal,ding2026riser}. This capability is appealing because recommendation is not only a question of \emph{what} a user consumed; it is also a question of \emph{which aspects} the user valued or rejected.

Historical reviews provide unusually direct evidence for this distinction \citep{deyoung2020eraser, paul2024making}. As shown in Figure \ref{fig:motivation}, consider a user with three previous audio-product interactions: \textbf{E1}, a five-star headphone review stating ``Very comfortable for long listening''; \textbf{E2}, a two-star earbud review stating ``The battery dies too quickly''; and \textbf{E3}, a five-star headset review praising microphone clarity. The next item is a comfortable, long-battery over-ear headset. An item-only trajectory can reveal that the user often consumes audio products. The reviews further suggest a positive preference for comfort and an avoidance preference for short battery life.


\begin{figure}[t]
  \centering
  \includegraphics[width=\columnwidth]{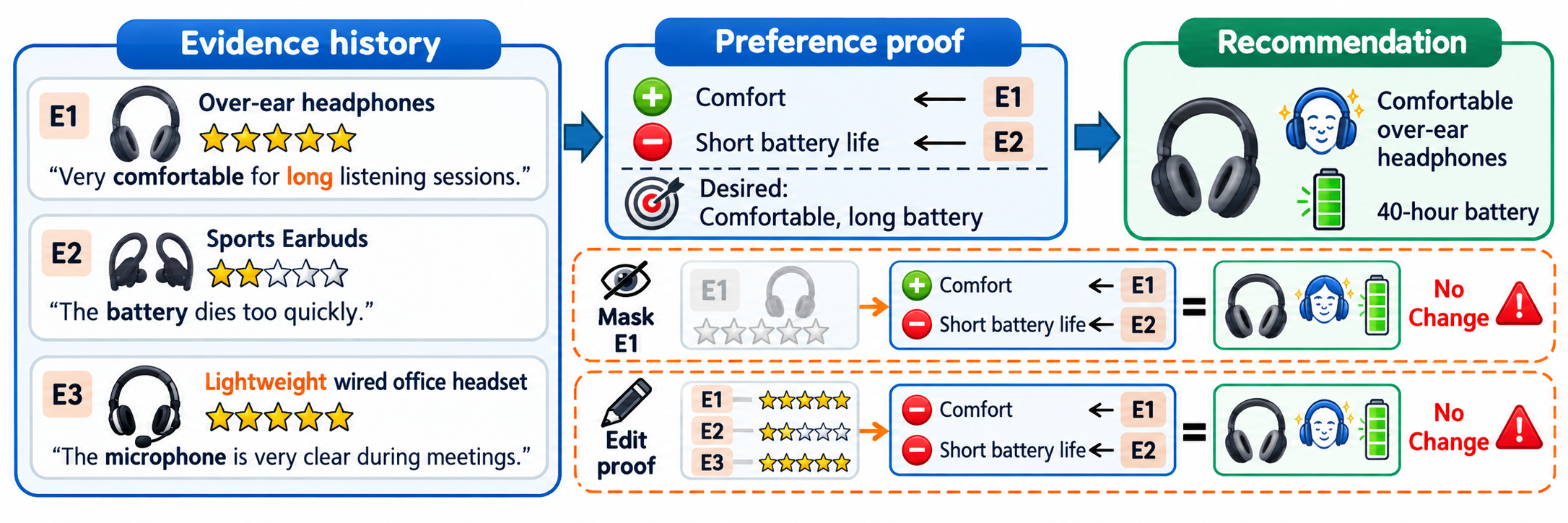}
  \Description{A single-column motivation diagram showing evidence history, preference-proof claims, and recommendation changes under evidence masking and preference editing.}
  \caption{\textbf{The grounding--influence gap through one running example.} \textbf{E1} and \textbf{E2} are selected by the preference proof, while \textbf{E3} serves as a comparable control entry. A rationale may appear plausible even if masking \textbf{E1} does not weaken the comfort claim, or if flipping the comfort-related preference does not reduce the target headset’s ranking advantage.}
  \label{fig:motivation}
\end{figure}

Plausibility, however, leaves two questions unanswered. First, does \textbf{E1} actually support the inferred preference for comfort, or is the evidence link merely decorative? Second, does the comfort claim affect the target ranking, or can the recommender ignore the proof and rely on category transitions, item popularity, or other shortcuts? As illustrated in Figure~\ref{fig:motivation}, the former failure occurs when masking \textbf{E1} barely changes the likelihood of the comfort claim, whereas the latter occurs when changing comfort from a positive preference to an avoidance preference leaves the target item's ranking margin nearly unchanged. Existing reasoning-enhanced recommenders primarily optimize recommendation accuracy, rationale likelihood, or trajectory-level reward \citep{liu2025onerecthink,he2026sidreasoner,zheng2026r2rank}. These objectives do not verify whether the stated evidence supports the preference or whether the preference affects the recommendation. Other LLM recommenders incorporate collaborative signals or reviews to improve recommendation quality \citep{zheng2024lcrec,zhang2026ragr}, while rationale-faithfulness research examines whether explanations reflect the evidence behind model predictions \citep{deyoung2020eraser}. However, these directions do not jointly verify both dependencies in LLM-based recommendation.

We use the term \textbf{grounding–influence gap} to describe two distinct operational dependencies. \textit{Evidence-to-proof grounding} asks whether the generation of a preference claim is selectively dependent on the historical evidence linked to that claim, relative to a comparable control entry. \textit{Proof-to-recommendation influence} asks whether the current ranking is locally sensitive to the claim when the selected evidence is held fixed. These two properties provide complementary views of the model’s reasoning behavior, but they do not define preference correctness. A semantically valid preference may be redundant with other claims or irrelevant to a particular target and therefore exhibit little measurable influence in that recommendation.

Our method is built on two complementary ideas.

\noindent\textbf{(\underline{Idea 1}): make \textit{evidence-to-proof grounding} explicit.}
Instead of producing unrestricted chain-of-thought text, the LLM generates a structured \emph{preference proof}: a compact set of positive or avoidance claims, each linked to one or more evidence identifiers. We use the term \emph{proof} operationally rather than logically; it denotes a checkable claim--evidence structure rather than a formal derivation. This structure exposes the \textit{evidence-to-proof} link to a controlled dependence test. For each claim, we mask its selected evidence and compare the resulting decrease in claim likelihood with that caused by masking a comparable control entry. A grounded claim should depend more strongly on its selected evidence than on the control entry.

\noindent\textbf{(\underline{Idea 2}): make \textit{proof-to-recommendation influence} explicit.}
The subsequent recommendation pass receives only the preference proof and the history entries selected by it. This proof-gated path prevents the model from generating a rationale while still predicting directly from the complete history. We then test the \emph{proof-to-recommendation} link by removing one preference claim while keeping the selected evidence and the remaining proof unchanged. If the claim is functionally used, its removal should reduce the target item's ranking margin. A format control changes only the presentation of the proof, separating semantic influence from generic prompt sensitivity. Finally, a listwise preservation objective keeps the proof-gated ranking distribution close to that obtained from the complete history.

The framework is not tied to a particular item representation. It applies to any autoregressive LLM recommender that can score the token sequence used to identify or describe an item. In our experiments, we instantiate PROVE-Rec on a strong generative recommender \citep{rajput2023tiger,zheng2024lcrec}; this choice defines the experimental implementation rather than the scope of the framework.

Our contributions are summarized as follows:
\begin{itemize}
\item We identify the \emph{grounding--influence gap} in LLM-based recommendation, where a generated preference rationale may be weakly supported by its selected evidence or have little effect on the final recommendation.
\item We introduce a compact preference proof and a proof-gated recommendation architecture that explicitly links preference claims to historical evidence and prevents the recommender from bypassing the generated proof.
\item We develop evidence-masking, claim-ablation, and ranking-preservation objectives to verify the evidence-to-proof and proof-to-recommendation links while maintaining recommendation utility.
\item We provide theoretical results that bound grounding and influence failures and characterize when the proof-gated model preserves the top-$K$ ranking of the full-history recommender.
\end{itemize}

The remainder of this paper is organized as follows. Section~\ref{sec:preliminaries} introduces the recommendation setting, evidence history, and preference-proof formulation. Section~\ref{sec:methodology} presents \method{}, including preference-proof induction, proof-gated recommendation, evidence-to-proof grounding, proof-to-recommendation influence, and ranking preservation. Section~\ref{sec:evaluation} describes the experimental setup and evaluates \method{} in terms of recommendation performance, component contributions, qualitative behavior, and direct verification of the evidence--proof--recommendation path. Finally, we summarize the main findings and conclude the paper.



\section{Related Work}
\label{sec:related-work}

\subsection{LLM-Based Recommendation}

Early work such as P5 reformulated multiple recommendation tasks as language-processing problems, enabling a shared sequence-to-sequence model to consume interaction histories, metadata, and natural-language prompts \citep{geng2022p5}. Subsequent methods adapted general-purpose LLMs to recommendation through parameter-efficient tuning, collaborative-semantic alignment, or explicit interfaces between behavioral representations and language models \citep{bao2023tallrec,liao2024llara,zheng2024lcrec,zhang2025collm}. More recently, RALLRec+ combines retrieval-augmented recommendation with LLM reasoning, further demonstrating the potential of integrating retrieved behavioral evidence with language-model inference \citep{luo2026rallrec}. These studies show that language knowledge and collaborative signals can complement each other for recommendation. However, they primarily focus on recommendation quality and model adaptation rather than directly examining whether an explicit preference claim selectively depends on particular historical evidence or contributes to the resulting ranking.

\subsection{Generative Recommendation and Item Tokenization}

Generative recommenders identify items through autoregressive token sequences instead of scoring every item independently. TIGER introduced semantic identifiers for generative retrieval, while LETTER learned item codes that better combine semantic and collaborative information \citep{rajput2023tiger,wang2024letter}. LC-Rec further integrates collaborative semantics into LLM adaptation, and DIGER studies differentiable semantic identifiers for end-to-end generative recommendation \citep{zheng2024lcrec,fu2026differentiable}. Beyond single-domain tokenization, UniTok develops a unified item-tokenization framework for multi-domain LLM-based recommendation, combining shared and domain-specific experts with codebook-based identifiers to preserve both transferable and domain-specific item semantics \citep{hou2026unitok}. These studies demonstrate the importance of item representations in connecting the discrete item space with autoregressive language models. In contrast, \method{} leaves the underlying item representation unchanged and instead modifies the information path through which historical evidence and explicit preference claims contribute to item prediction.

\subsection{Review-Aware and Reasoning-Enhanced Recommendation}

Reviews provide fine-grained signals about user preferences and item characteristics that are often absent from item-only interaction sequences. Prior review-aware recommenders have used deep Transformer architectures to extract textual preference signals and improve personalized recommendation \citep{gheewala2024transformer}. AENAR further models aspect-level information from reviews to capture diverse user preferences and provide interpretable evidence for rating prediction \citep{zhang2022aenar}. More recently, RAGR directly incorporates review representations into generative user sequences and aligns the resulting representations with item prediction \citep{zhang2026ragr}. In contrast, \method{} organizes historical reviews into a compact set of evidence-linked preference claims and explicitly examines whether each claim depends selectively on its cited evidence.

Recent reasoning-enhanced recommenders generate explicit or latent reasoning before predicting an item. OneRec-Think introduces in-text reasoning for generative recommendation, SIDReasoner reasons over semantic identifiers, R2Rank optimizes reasoning for ranking utility, and RE2 studies reasoning for more interpretable recommendation predictions \citep{liu2025onerecthink,he2026sidreasoner,zheng2026r2rank,chen2026prediction}. RALLRec+ additionally combines retrieval augmentation with reasoning LLMs to enhance recommendation \citep{luo2026rallrec}. These methods demonstrate that intermediate reasoning can improve recommendation, but recommendation accuracy or trajectory-level objectives alone do not directly reveal how individual reasoning statements depend on historical evidence or affect a particular ranking decision. \method{} focuses on these two complementary dependencies through evidence-to-proof grounding and proof-to-recommendation influence.

\subsection{Rationale Grounding and Intervention-Based Evaluation}

A fluent rationale can appear plausible without revealing how it relates to the evidence or the model's final prediction. ERASER evaluates rationales through evidence-based metrics and perturbation tests \citep{deyoung2020eraser}. Faithful Chain-of-Thought makes an intermediate symbolic reasoning representation necessary for producing the final answer, while later work examines whether generated reasoning traces have measurable effects on model predictions \citep{lyu2023faithful,paul2024making}. Related work in recommender systems has also explored post-hoc explainability through deletion-based diagnostics that measure how individual user or item information affects recommendation behavior \citep{arevalo2026modelagnostic}. These studies motivate intervention-based evaluation, but they do not jointly address recommendation-specific questions such as preference polarity, evidence selection from sequential user histories, claim-level ranking effects, and preservation of recommendation utility. \method{} carries this intervention-based perspective into LLM-based recommendation with two paired tests on the evidence--proof--recommendation path: whether masking a claim's linked evidence lowers its score more than masking a comparable control entry, and whether removing the claim lowers the target item's ranking margin while the selected evidence remains fixed.

\section{Preliminaries}
\label{sec:preliminaries}


\subsection{Task and Evidence History}
Let $\mathcal U$ and $\mathcal I$ denote the user and item sets. For user $u$ at prediction time $t$, the ordered evidence history observed before the target interaction is
\begin{equation}
\mathcal H_u^t=[E_{u,1},\ldots,E_{u,n_u}],
\label{eq:history}
\end{equation}
where $n_u=|\mathcal H_u^t|$ is the number of historical entries. The $j$-th \emph{evidence entry} is
\begin{equation}
E_{u,j}=(i_{u,j},r_{u,j},v_{u,j}),
\label{eq:evidence-entry}
\end{equation}
where $i_{u,j}\in\mathcal I$ is the historical item, $r_{u,j}\in\{1,\ldots,5\}\cup\{\varnothing\}$ is its star rating, with $\varnothing$ marking a missing rating, and $v_{u,j}$ is its review text. Each item $i$ also carries a category label $\kappa(i)$ from its metadata. The ordering of $\mathcal H_u^t$ preserves chronology, so the position $j$ encodes recency and timestamps are omitted. Item titles and other item text may appear in the prompt but are omitted from the notation.

Within each training example, the entries are identified in the prompt as \textbf{E1}, \textbf{E2}, and so on. The identifier is retained in every counterfactual view of that example, even when the entry's content is masked. The ground-truth next item is $y_u=i_{u,t}$. Neither the target review nor any future interaction is included in $\mathcal H_u^t$.

An LLM recommender represents item $i$ by an output token sequence $\mathbf c(i)=[c_{i,1},\ldots,c_{i,L_c}]$, where $L_c$ is the output length \citep{rajput2023tiger,wang2024letter}. The representation can be a textual identifier, a learned item token sequence, or a semantic code. With model parameters $\theta$, the normalized teacher-forced item score under a context $X$ is
\begin{equation}
 s_\theta(i\mid X)=\frac{1}{L_c}\sum_{\ell=1}^{L_c}
 \log p_\theta(c_{i,\ell}\mid X,c_{i,<\ell}).
 \label{eq:item-score}
\end{equation}

\subsection{Preference Proof}
For each user \(u\), we represent the preference proof as a compact sequence of \(m_u\) structured claims:
\begin{equation}
P_u=[C_{u,1},\ldots,C_{u,m_u}],
\qquad
C_{u,k}=(\sigma_{u,k},a_{u,k},S_{u,k}),
\label{eq:proof}
\end{equation}
where \(\sigma_{u,k}\in\{+1,-1\}\) is the polarity of the claim, rendered in the proof text as \emph{prefer} for \(+1\) and \emph{avoid} for \(-1\), \(a_{u,k}\) is a normalized preference aspect, and \(S_{u,k}\subseteq\{1,\ldots,n_u\}\) contains the indices of the evidence entries cited for the claim. For example,
\[
C_{u,1}=(+1,\text{comfort},\{1\})
\]
indicates that the first claim for user $u$ expresses a positive preference for ``comfort'' and is supported by evidence entry \textbf{E1}.

The complete set of evidence selected by the proof is
\begin{equation}
S_u=\bigcup_{k=1}^{m_u}S_{u,k}.
\label{eq:selected-set}
\end{equation}
Evidence entries indexed by \(S_u\) are \emph{selected evidence}; the remaining entries, indexed by \(U_u=\{1,\ldots,n_u\}\setminus S_u\), are \emph{unselected evidence}. An unselected entry is not assumed to be irrelevant; the current proof simply does not cite it.

To keep the proof concise and auditable, we restrict both the number of claims and the number of selected evidence entries:
\begin{equation}
1\leq m_u\leq K_C,
\qquad
1\leq |S_u|\leq K_E,
\label{eq:compact-proof}
\end{equation}
where \(K_C\) and \(K_E\) are the claim and evidence budgets. The data-validation filter of Section~\ref{sec:evaluation} enforces the same two budgets, written there as \(C_{\max}=K_C\) and \(E_{\max}=K_E\). Unselected entries supply the control entries of Section~\ref{sec:grounding}, so a history in which every entry is selected (\(U_u=\varnothing\)) contributes no grounding term.

\section{Methodology}
\label{sec:methodology}


\begin{figure*}[t]
  \centering
  \includegraphics[width=0.99\textwidth]{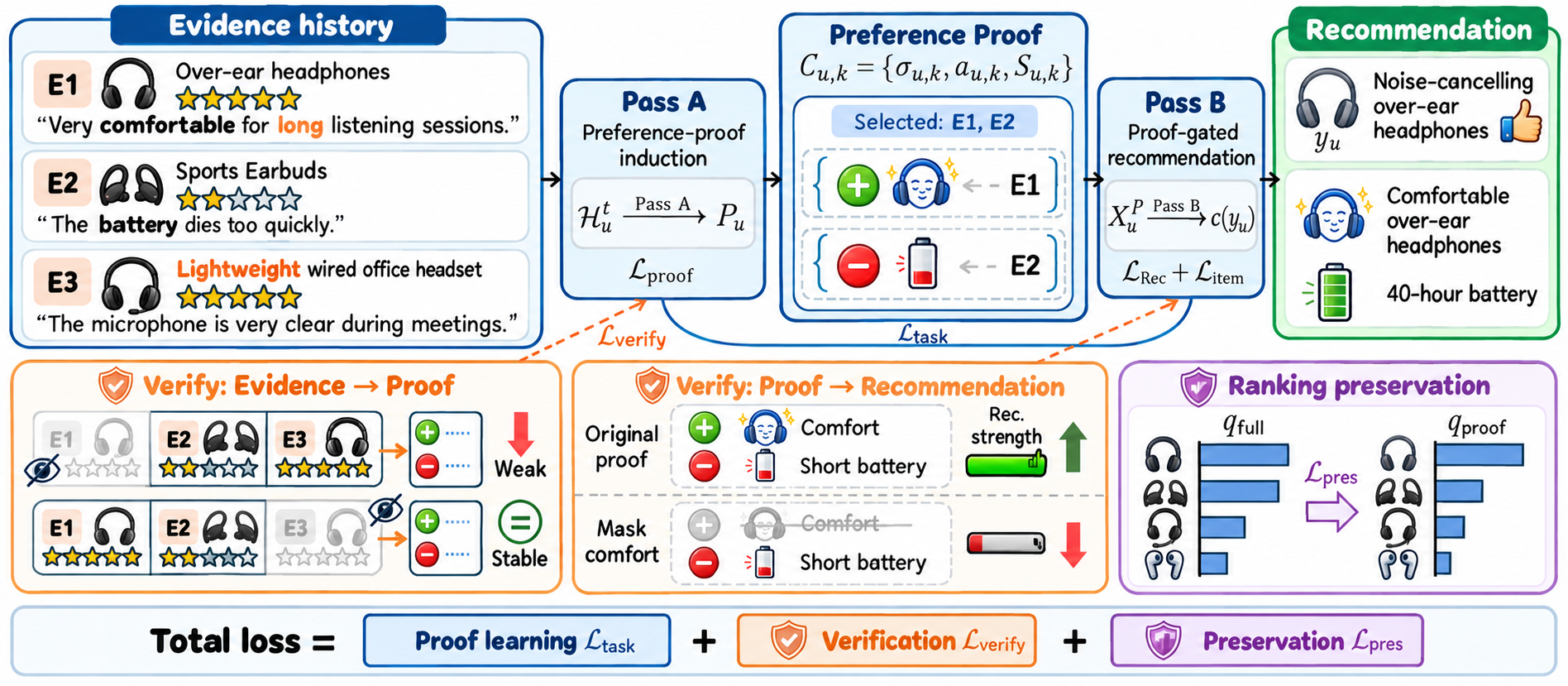}
  \Description{Overview of PROVE-Rec with evidence history, preference-proof induction, proof-gated recommendation, grounding and influence verification, and ranking preservation.}
  \caption{\textbf{Overview of \method{} on the audio example of Figure~\ref{fig:motivation}.} Pass A converts the full history into a compact preference proof that selects \textbf{E1} and \textbf{E2}. Pass B predicts the next item from the proof and the selected evidence only. The grounding check compares the comfort-claim score after masking the selected entry \textbf{E1} with the score after masking the control entry \textbf{E3}. The influence check deletes the comfort claim, keeps \textbf{E1} and \textbf{E2} in the prompt, and measures the change in the target ranking margin. A preservation term aligns the proof-gated ranking distribution with the distribution the same model assigns from the full history.}
  \label{fig:overview}
\end{figure*}

\subsection{Method Overview}
Figure~\ref{fig:overview} follows \method{} through the audio-product example. The evidence history $\mathcal H_u^t$ contains three entries: a positive review praising the comfort of over-ear headphones (\textbf{E1}), a low-rated review criticizing the short battery life of sports earbuds (\textbf{E2}), and a positive review praising the microphone of an office headset (\textbf{E3}). Pass A generates a preference proof $P_u$ with two claims: the user prefers comfort, citing \textbf{E1}, and avoids short battery life, citing \textbf{E2}. Pass B then uses only the proof and the selected entries \textbf{E1} and \textbf{E2} to generate the next item $y_u$, a comfortable over-ear headset with long battery life.

The same LLM is used in both passes:
\begin{equation}
\mathcal H_u^t
\xrightarrow{\text{Pass A}}
P_u,
\qquad
X_u^P
\xrightarrow{\text{Pass B}}
\mathbf c(y_u),
\label{eq:two-pass}
\end{equation}
where $\mathbf c(y_u)$ is the token sequence representing the target item. The input to Pass B is the proof-gated context
\begin{equation}
X_u^P=\phi \left(P_u,\mathcal H_u^t[S_u]\right),
\label{eq:proof-context}
\end{equation}
where $\mathcal H_u^t[S_u]$ keeps the selected entries in chronological order. The function $\phi$ inserts the proof and the selected evidence into a fixed prompt template and has no trainable parameters.

The two-pass model factorizes the joint prediction as
\begin{equation}
 p_\theta(P_u,\mathbf c(y_u)\mid\mathcal H_u^t)
 =p_\theta(P_u\mid\mathcal H_u^t)
 p_\theta(\mathbf c(y_u)\mid X_u^P).
\label{eq:factorization}
\end{equation}
Training checks one link in each factor: whether the selected evidence drives the claims of the proof (Section~\ref{sec:grounding}), and whether the claims drive the recommendation (Section~\ref{sec:influence}). A ranking-preservation term (Section~\ref{sec:preservation}) limits the information lost by restricting Pass B to the proof-gated context.

\subsection{Preference-Proof Induction}
Pass A maps the complete pre-target evidence history \(\mathcal H_u^t\) to a structured preference proof \(P_u\), rendered in a fixed textual format. We train Pass A with the token-normalized generation loss
\begin{equation}
\mathcal L_{\mathrm{proof}}
=
-\frac{1}{|P_u|}
\log p_\theta(P_u\mid\mathcal H_u^t),
\label{eq:proof-loss}
\end{equation}
where \(|P_u|\) is the number of tokens in the rendered proof. Positive and avoidance claims share one format and differ only in \(\sigma_{u,k}\), so a low-rated review supports an avoidance claim in the same way that a high-rated review supports a positive one.

Pass A is supervised with offline preference proofs generated from the pre-target history only. A proof is retained only if it passes automatic checks on format, evidence identifiers, polarity labels, the budgets in Eq.~\eqref{eq:compact-proof}, and temporal leakage (Section~\ref{sec:data-preprocessing}). These checks ensure that every claim cites valid pre-target entries within budget. Whether the cited content actually drives the claim inside the model is what the grounding contrast of Section~\ref{sec:grounding} measures.

\subsection{Proof-Gated Recommendation}

Pass B predicts the next item from the proof-gated context \(X_u^P\), which contains only the preference proof and its selected evidence entries. The rest of the history is not accessible to Pass B. The token-level generation loss for the target item \(y_u\) is
\begin{equation}
\mathcal L_{\mathrm{rec}} = -\frac{1}{L_c} \log p_\theta\!\left(\mathbf c(y_u)\mid X_u^P\right).
\label{eq:item-loss}
\end{equation}

For item-level ranking, let \(\mathcal N_u\) be a set of \(N\geq 1\) sampled non-target items and \(\mathcal C_u=\{y_u\}\cup\mathcal N_u\) the comparison set, so that \(|\mathcal C_u|=N+1\). For any context \(X\), the candidate distribution over \(\mathcal C_u\) is
\begin{equation}
q_\theta(i\mid X,\mathcal C_u)
=
\frac{\exp\!\left(s_\theta(i\mid X)/\tau_r\right)}
{\sum_{h\in\mathcal C_u}
\exp\!\left(s_\theta(h\mid X)/\tau_r\right)},
\qquad i\in\mathcal C_u,
\label{eq:candidate-dist}
\end{equation}
where \(s_\theta\) is the item score of Eq.~\eqref{eq:item-score} and \(\tau_r>0\) is a ranking temperature. The proof-gated item distribution is
\begin{equation}
q_{\mathrm{proof}}(i)
=
q_\theta(i\mid X_u^P,\mathcal C_u),
\label{eq:qproof}
\end{equation}
and the item loss is
\begin{equation}
\mathcal L_{\mathrm{item}}
=
-\log q_{\mathrm{proof}}(y_u).
\label{eq:rec-loss}
\end{equation}

At inference, Pass B performs constrained generation over the full item catalog but still receives only \(X_u^P\), so unselected history cannot enter the recommendation through a side channel.

\subsection{Evidence-to-Proof: \emph{Does the Claim Come from the Selected Evidence?}}
\label{sec:grounding}

The first check asks whether a claim depends on the entries cited for it. For each claim $C_{u,k}$, the normalized teacher-forced claim score under a history $\mathcal H$ is
\begin{equation}
g_\theta(C_{u,k}\mid\mathcal H)
=
\frac{1}{\operatorname{len}(C_{u,k})}
\log p_\theta
\left(
C_{u,k}\mid\mathcal H,P_{u,<k}
\right),
\label{eq:claim-score}
\end{equation}
where $\operatorname{len}(C_{u,k})$ is the number of tokens in the claim and $P_{u,<k}$ contains the preceding claims in the fixed proof order.

\paragraph{Masking and control entries}
Let $\mathcal M_j(\mathcal H_u^t)$ denote the history obtained by replacing the content of $E_{u,j}$ (item text, rating, and review) with a fixed mask token while keeping its identifier and position. For a link $(k,j)$ with $j\in S_{u,k}$, the control entry is taken from the unselected positions $U_u$. We group ratings by polarity,
\begin{equation}
\pi(r)=
\begin{cases}
+1, & r\geq 4,\\
-1, & r\leq 2,\\
0, & r=3 \text{ or } r=\varnothing,
\end{cases}
\label{eq:rating-polarity}
\end{equation}
and choose the unselected entry of the same polarity that is nearest in position,
\begin{equation}
c(j)=\operatorname*{arg\,min}_{j'\in U_u:\ \pi(r_{u,j'})=\pi(r_{u,j})}|j-j'|,
\label{eq:control-map}
\end{equation}
breaking ties first in favor of an entry whose item shares the category $\kappa(i_{u,j})$ and then in favor of the earlier position. If no unselected entry shares the polarity of $E_{u,j}$, the minimization in Eq.~\eqref{eq:control-map} runs over all of $U_u$. When a history carries no ratings, every entry has polarity $0$ and the rule reduces to recency with category tie-breaking. A link therefore lacks a control only when $U_u$ is empty, and the valid grounding links of user $u$ are
\begin{equation}
\mathcal G_u=\left\{(k,j):\ 1\leq k\leq m_u,\ j\in S_{u,k},\ U_u\neq\varnothing\right\}.
\label{eq:valid-grounding}
\end{equation}

\paragraph{Grounding contrast}
For each $(k,j)\in\mathcal G_u$, the grounding contrast is
\begin{equation}
\Delta_{\mathrm{grd}}(k,j)
=
g_\theta
\left(
C_{u,k}\mid
\mathcal M_{c(j)}(\mathcal H_u^t)
\right)
-
g_\theta
\left(
C_{u,k}\mid
\mathcal M_j(\mathcal H_u^t)
\right).
\label{eq:grounding-contrast}
\end{equation}
Both terms are scored on masked histories, so the contrast compares the selected entry with its control rather than with the intact history. A positive value means that hiding the selected entry lowers the claim score more than hiding the control. A non-positive value, ties included, is a grounding failure. In Figure~\ref{fig:overview}, masking the over-ear headphone review (\textbf{E1}) should lower the comfort-claim score more than masking the office-headset review (\textbf{E3}). The control is not assumed to be irrelevant to the user; it supplies a local baseline for the perturbation that any mask causes.

The grounding loss averages a softplus penalty over the valid links,
\begin{equation}
\mathcal L_{\mathrm{grd}}
=
\frac{1}{|\mathcal G_u|}
\sum_{(k,j)\in\mathcal G_u}
\operatorname{softplus}
\left(
-\frac{\Delta_{\mathrm{grd}}(k,j)}{\tau_g}
\right),
\label{eq:ground-loss}
\end{equation}
where $\operatorname{softplus}(x)=\log(1+\exp(x))$ and $\tau_g>0$ sets the sharpness of the comparison. Users with $\mathcal G_u=\varnothing$ contribute no grounding term.

\subsection{Proof-to-Recommendation: \emph{Does the Claim Affect the Recommendation?}}
\label{sec:influence}

The second check asks whether a claim moves the recommendation. For claim $C_{u,k}$, the claim-deleted proof $P_u^{\setminus k}$ removes $C_{u,k}$ and keeps the other claims in order. The selected evidence set $S_u$ is unchanged, which gives the claim-deleted context
\begin{equation}
X_u^{\setminus k}
=
\phi\!\left(P_u^{\setminus k},\mathcal H_u^t[S_u]\right).
\label{eq:ablated-context}
\end{equation}
When $m_u=1$, $P_u^{\setminus k}$ is the empty proof, which the template renders as an empty claim list, so Pass B sees only the selected evidence. Because $S_u$ is held fixed, the contrast measures what the claim text adds beyond the evidence it cites.

For a proof-gated context $X$, the target ranking margin is the gap between the target and its strongest sampled competitor,
\begin{equation}
M_\theta(X)
=
s_\theta(y_u\mid X)
-
\max_{i\in\mathcal N_u}
s_\theta(i\mid X).
\label{eq:margin}
\end{equation}
The margin is measured on the sampled slate $\mathcal C_u$ rather than on the full-catalog ranking produced at inference. The influence contrast of claim $k$ is
\begin{equation}
\Delta_{\mathrm{inf}}(k)
=
M_\theta(X_u^P)
-
M_\theta(X_u^{\setminus k}),
\label{eq:influence-contrast}
\end{equation}
and a positive value means that deleting the claim reduces the target's lead. A non-positive value, ties included, is an influence failure. The influence loss averages over the claims of the proof,
\begin{equation}
\mathcal L_{\mathrm{inf}}
=
\frac{1}{m_u}
\sum_{k=1}^{m_u}
\operatorname{softplus}
\left(
-\frac{\Delta_{\mathrm{inf}}(k)}{\tau_p}
\right),
\label{eq:influence-loss}
\end{equation}
where $\tau_p>0$ controls the smoothness of the pairwise objective.

The loss asks for a smaller margin rather than a change of the top-1 item because several claims usually support the same target. A correct claim that repeats another claim can still show a small $\Delta_{\mathrm{inf}}(k)$. The loss counts it as a failure, which pushes Pass B to draw on each retained claim rather than on one of them.

\subsection{Ranking Preservation}
\label{sec:preservation}

The proof-gated context keeps the recommendation path inspectable, but it can drop useful signals from the unselected history. We therefore score the same comparison set from the full-history prompt with the same parameters,
\begin{equation}
q_{\mathrm{full}}(i)
=
q_\theta(i\mid \mathcal H_u^t,\mathcal C_u),
\label{eq:qfull}
\end{equation}
and align the proof-gated distribution with it:
\begin{equation}
\mathcal L_{\mathrm{pres}}
=
D_{\mathrm{KL}}
\left(
\sg[q_{\mathrm{full}}]
\,\|\,
q_{\mathrm{proof}}
\right),
\label{eq:preserve}
\end{equation}
where \(\sg[\cdot]\) stops gradients through \(q_{\mathrm{full}}\). The distribution \(q_{\mathrm{full}}\) is not a separately trained full-history recommender. No item loss is applied to the full-history prompt, and the stop-gradient prevents \(\mathcal L_{\mathrm{pres}}\) from moving \(q_{\mathrm{full}}\) toward \(q_{\mathrm{proof}}\); it is a passive probe of the ranking the current parameters would give if Pass B saw the complete history. Both distributions use the same comparison set, scoring function, and temperature, so \(\mathcal L_{\mathrm{pres}}\) constrains relative item order without asking the two contexts to share representations. The term controls agreement, while ranking accuracy comes from \(\mathcal L_{\mathrm{rec}}\) and \(\mathcal L_{\mathrm{item}}\) on the proof-gated context.

\subsection{Training Objective}
The per-example objectives are grouped by role:
\begin{align}
\mathcal L_{\mathrm{task}}
&=\tfrac{1}{3}(\mathcal L_{\mathrm{proof}}
+\mathcal L_{\mathrm{item}}
+\mathcal L_{\mathrm{rec}}),
\label{eq:task-loss}\\
\mathcal L_{\mathrm{verify}}
&=\tfrac{1}{2}(\mathcal L_{\mathrm{grd}}
+\mathcal L_{\mathrm{inf}}),
\label{eq:verify-loss}
\end{align}
and the final objective is
\begin{equation}
\boxed{
\mathcal L
=\mathcal L_{\mathrm{task}}
+\lambda\mathcal L_{\mathrm{verify}}
+\mu\mathcal L_{\mathrm{pres}},}
\label{eq:overall}
\end{equation}
where $\lambda\geq 0$ sets the verification strength and $\mu\geq 0$ the strength of ranking preservation. For a user without valid grounding links, $\mathcal L_{\mathrm{verify}}$ reduces to $\tfrac12\mathcal L_{\mathrm{inf}}$. Training proceeds in three phases. We first train Pass A to generate preference proofs from evidence histories, then train Pass B to predict the target item from the proof-gated context, and finally optimize Eq.~\eqref{eq:overall} jointly with the masked-evidence and claim-deleted views defined above.

\subsection{What the Theory Guarantees}
\label{sec:theory}
Both results are stated for the quantities defined above, and their complete proofs are in Appendix~\ref{app:theory}. Let $\nu_{\mathrm{grd}}$ be the distribution that draws a training user $u$ with $\mathcal G_u\neq\varnothing$ uniformly and then a link $(k,j)$ uniformly from $\mathcal G_u$, and let $\nu_{\mathrm{inf}}$ draw a training user uniformly and then a claim $k$ uniformly from $\{1,\ldots,m_u\}$. Under these distributions, the expected softplus penalty equals the user-averaged $\mathcal L_{\mathrm{grd}}$ of Eq.~\eqref{eq:ground-loss} and $\mathcal L_{\mathrm{inf}}$ of Eq.~\eqref{eq:influence-loss}, respectively.

\begin{theorem}[Two-link failure-rate control]
\label{thm:failure}
Let $\tau_g,\tau_p>0$. If
\begin{equation}
\begin{aligned}
\mathbb E_{\nu_{\mathrm{grd}}}\!\left[\operatorname{softplus}\!\left(-\Delta_{\mathrm{grd}}/\tau_g\right)\right]
&\leq\epsilon_{\mathrm{grd}},\\
\mathbb E_{\nu_{\mathrm{inf}}}\!\left[\operatorname{softplus}\!\left(-\Delta_{\mathrm{inf}}/\tau_p\right)\right]
&\leq\epsilon_{\mathrm{inf}}.
\end{aligned}
\label{eq:failure-hyp}
\end{equation}
then
\begin{equation}
\Pr_{\nu_{\mathrm{grd}}}(\Delta_{\mathrm{grd}}\leq0)
\leq\frac{\epsilon_{\mathrm{grd}}}{\log 2},
\qquad
\Pr_{\nu_{\mathrm{inf}}}(\Delta_{\mathrm{inf}}\leq0)
\leq\frac{\epsilon_{\mathrm{inf}}}{\log 2}.
\label{eq:failure-bound}
\end{equation}
\end{theorem}

\noindent\emph{Proof sketch.} On a failure event the softplus argument is non-negative, so the penalty is at least $\log 2$. The failure indicator is therefore at most the penalty divided by $\log 2$, and taking expectations gives Eq.~\eqref{eq:failure-bound}. The complete proof is in~\ref{app:proof_failure}. \hfill$\square$

A grounding failure is a link whose selected-entry mask hurts the claim no more than the control mask, and an influence failure is a claim whose deletion leaves the target's lead on the sampled slate intact. Theorem~\ref{thm:failure} turns an average verification loss into a ceiling on how often either happens, and the ceiling constrains training once $\epsilon$ falls below $\log 2$. Because both views in $\Delta_{\mathrm{grd}}$ are masked, the bound concerns the selected entry relative to its control, not either mask relative to the intact history.

\begin{theorem}[Top-$K$ preservation]
\label{thm:topk}
Let $q_{\mathrm{full}}$ and $q_{\mathrm{proof}}$ be the distributions of Eqs.~\eqref{eq:qfull} and~\eqref{eq:qproof} on the comparison set $\mathcal C_u$ of one user, and let $1\leq K<|\mathcal C_u|$. Order the items so that $q_{\mathrm{full}}(i_{(1)})\geq\cdots\geq q_{\mathrm{full}}(i_{(|\mathcal C_u|)})$, and define the boundary margin
\begin{equation}
\gamma_K
=
q_{\mathrm{full}}(i_{(K)})
-
q_{\mathrm{full}}(i_{(K+1)}).
\label{eq:boundary-margin}
\end{equation}
If $D_{\mathrm{KL}}(q_{\mathrm{full}}\|q_{\mathrm{proof}})\leq\delta$ and $\gamma_K>\sqrt{2\delta}$, then the $K$ most probable items form the same set under both distributions,
\begin{equation}
\TopK(q_{\mathrm{full}})
=
\TopK(q_{\mathrm{proof}}).
\label{eq:topk-equal}
\end{equation}
\end{theorem}

\noindent\emph{Proof sketch.} Pinsker's inequality \citep{canonne2022short} bounds the total-variation distance, and hence every coordinate difference $|q_{\mathrm{full}}(i)-q_{\mathrm{proof}}(i)|$, by $\sqrt{\delta/2}$. Two such perturbations cannot close a gap larger than $\sqrt{2\delta}$, so every item above the boundary under $q_{\mathrm{full}}$ stays above every item below it under $q_{\mathrm{proof}}$. The complete proof is in~\ref{app:proof_topk}. \hfill$\square$

Since $\gamma_K>0$, the top-$K$ set under $q_{\mathrm{full}}$ is unique, and the conclusion makes it unique under $q_{\mathrm{proof}}$ as well. The per-example value of $\mathcal L_{\mathrm{pres}}$ is exactly the divergence in the hypothesis, so Theorem~\ref{thm:topk} is a condition on the preservation loss: where the full-history ranking separates ranks $K$ and $K+1$ clearly, a small $\mathcal L_{\mathrm{pres}}$ keeps the proof-gated ranking of the slate on the same top-$K$ set. The guarantee lives on the sampled comparison set $\mathcal C_u$. Full-catalog Recall and NDCG in Section~\ref{sec:evaluation} are measured separately and fall under the theorem only when $\mathcal C_u$ is the whole catalog.

\section{Experimental Evaluation}
\label{sec:evaluation}

We conduct comprehensive experiments to evaluate the recommendation effectiveness of \method{}, analyze the contribution of its main components, and directly examine the behavior of the generated preference proofs. Specifically, we investigate the following five research questions (\textbf{RQs}):

\begin{itemize}
    \item {\bf RQ1}: Does verifiable preference reasoning improve recommendation performance?
    \item {\bf RQ2}: How do the main components of \method{} contribute to recommendation performance?
    \item {\bf RQ3}: Which information sources and reasoning designs are most effective?
    \item {\bf RQ4}: How does \method{} select evidence, construct preference proofs, and support recommendations in a qualitative case study?
    \item {\bf RQ5}: Do the verification objectives strengthen evidence dependence and decision relevance of the generated preference proofs?
\end{itemize}


\subsection{Datasets, Preprocessing, and Preference-Proof Construction}
\label{sec:data-preprocessing}

We evaluate \method{} on on two Amazon review domains \citep{ni2019amazon}, \textit{Grocery} and \textit{Tools}, together with \textit{Yelp}\footnote{https://www.yelp.com/dataset/.} All datasets contain user--item interactions together with textual reviews, enabling us to construct evidence-enriched interaction histories and preference proofs. Table~\ref{tab:dataset-stats} summarizes the processed datasets.


\begin{table}[t]
\centering
\scriptsize
\setlength{\tabcolsep}{2pt}
\caption{Statistics of the processed datasets. Density is
$100|\mathcal{R}|/(|\mathcal{U}||\mathcal{I}|)$, and review coverage denotes
the percentage of interactions with nonempty aligned review text.}
\label{tab:dataset-stats}
\begin{tabular}{lrrrrrr}
\toprule
Dataset & Users & Items & Interactions &
Density (\%) & Avg. length & Review cov. (\%) \\
\midrule
Grocery & 14{,}681 & 8{,}713  & 151{,}254 & 0.1182 & 10.30 & 99.99 \\
Tools   & 16{,}638 & 10{,}217 & 134{,}476 & 0.0791 &  8.08 & 99.95 \\
Yelp    & 30{,}431 & 20{,}033 & 316{,}354 & 0.0519 & 10.40 & 98.87 \\
\bottomrule
\end{tabular}
\end{table}

\paragraph{Data preprocessing}
For each user, interactions are ordered chronologically. The last interaction is used for testing, the second-to-last for validation, and the remaining interactions for training. Each model input contains only interactions that occur before the corresponding target, and previously interacted items are excluded from the recommendation candidates during evaluation.

Each example retains at most 20 pre-target interactions. Ratings and reviews are included for the three most recent entries, with review text truncated to 64 characters and review summaries to 32 characters. Reviews are aligned with the processed interaction sequences using user and item identifiers, while ambiguous duplicate matches are excluded. The target item, its rating and review, and all future interactions are removed from both Pass A and Pass B to prevent temporal leakage.

\paragraph{Preference-proof construction.}
Pass A is supervised using offline preference proofs constructed exclusively from the pre-target history. Each proof consists of at most $K_C$ preference claims and selects at most $K_E$ historical evidence entries. Each claim contains a normalized preference aspect, a polarity (\textsc{prefer} or \textsc{avoid}), and one or more identifiers of supporting history entries.

Candidate proofs are retained only if they satisfy the required output format, contain valid and unique evidence identifiers, respect the claim and evidence budgets, and contain no target or future information. We additionally check that every cited identifier refers to an entry in the corresponding pre-target history. These filters ensure structural validity, referential consistency, budget compliance, and explicit leakage prevention; they are not treated as tests of semantic entailment between a claim and its cited evidence. Table~\ref{tab:proof-stats} reports the resulting proof statistics.

\begin{table}[t]
\centering
\small
\setlength{\tabcolsep}{5pt}
\caption{Statistics of the offline preference proofs. Proofs are constructed
from histories of at most 20 interactions with an evidence budget of
$K_E=3$. Acceptance rate denotes the percentage of candidate proofs retained
after structural and leakage validation.}
\label{tab:proof-stats}
\begin{tabular}{lrrrr}
\toprule
Dataset & Proofs & Claims/proof & Selected/proof & Accept rate \\
\midrule
Grocery & 136{,}559 & 1.189 & 2.409 & 99.990\% \\
Tools   & 117{,}829 & 1.145 & 2.398 & 99.992\% \\
Yelp    & 274{,}625 & 1.000 & 2.675 & 96.049\% \\
\bottomrule
\end{tabular}
\end{table}

\subsection{Experimental Setup}


\paragraph{Implementation Details}
Our implementation uses a semantic-code generative recommender built on TIGER and LC-Rec \citep{rajput2023tiger,zheng2024lcrec}. For each example, we retain up to 20 historical evidence entries and include ratings and reviews for the three most recent entries. Review text and item metadata are truncated to fixed lengths. We preserve the original item tokenizer, recommendation objective, and catalog-constrained decoding procedure. Additional implementation and hyperparameter details are provided in the supplementary material.

\begin{table*}[!t]\centering
\setlength\tabcolsep{4.5pt}
\small

\begin{tabular}{l|cccc|cccc|cccc}
\toprule[1pt]
\multicolumn{1}{c|}{} & \multicolumn{4}{c|}{\bf Grocery} & \multicolumn{4}{c|}{\bf Tools} & \multicolumn{4}{c}{\bf Yelp} \\
\cmidrule{2-13}
{\bf Method} & {\bf R@5} & {\bf R@10} & {\bf N@5} & {\bf N@10}
& {\bf R@5} & {\bf R@10} & {\bf N@5} & {\bf N@10}
& {\bf R@5} & {\bf R@10} & {\bf N@5} & {\bf N@10} \\
\midrule[1pt]

\multicolumn{13}{l}{\textit{Traditional and Transformer-based Recommenders}} \\
\midrule
{\bf SASRec} &0.0359 &0.0509 &0.0183 &0.0276 & 0.0135& 0.0175&0.0117 &0.0134 &0.0163 &0.0354 & 0.0122& 0.0183\\
{\bf BERT4Rec} & 0.0294& 0.0448 & 0.0187&0.0237 & 0.0106&0.0184 &0.0067 &0.0092 &0.0155 &0.0272 & 0.0094& 0.0131\\

\cmidrule{1-13}
\multicolumn{13}{l}{\textit{Item-tokenization-based LLM Recommenders}} \\
\midrule
{\bf P5-SemID} &0.0367 &0.0524 & 0.0201&0.0285 & 0.0177&0.0297 &0.0108 &0.0133 & 0.0187&0.0324 &0.0129 &0.0188 \\
{\bf TIGER} &0.0423 & 0.0635& 0.0273& 0.0341& 0.0216&0.0327 &0.0139 & 0.0175 &0.0241 &\underline{0.0397} & 0.0155 & \underline{0.0205}\\
{\bf LC-Rec} &0.0437 & 0.0622& \underline{0.0281} &\underline{0.0349} & 0.0232 & 0.0351&0.0148 & 0.0207& 0.0232&0.0388 &0.0147 & 0.0195\\
{\bf DIGER} & \underline{0.0440}& 0.0641& 0.0279& 0.0345& \underline{0.0241}& \underline{0.0355}& \underline{0.0153}& \underline{0.0212}& 0.0247& 0.0391& \underline{0.0161}& 0.0199\\

\cmidrule{1-13}
\multicolumn{13}{l}{\textit{LLM-enhanced Recommenders}} \\
\midrule
{\bf CoLLM}
& 0.0398 & 0.0619
& \underline{0.0281} & 0.0340
& 0.0203 & 0.0315
& 0.0142 & 0.0181
& 0.0237 & 0.0389
& 0.0158 & 0.0202 \\

{\bf LLM-SRec}
& 0.0431 & \underline{0.0642}
& 0.0268 & 0.0347
& 0.0221 & 0.0334
& 0.0136 & 0.0173
& \underline{0.0248} & 0.0395
& 0.0153 & 0.0198 \\
{\bf SIDReasoner} & 0.0415& 0.0625& 0.0274& 0.0329& 0.0214& 0.0321& 0.0129& 0.0165& 0.0231 & 0.0377& 0.0146& 0.0184\\
{\bf RE2} & 0.0377 & 0.0602& 0.0266& 0.0322& 0.0209& 0.0327& 0.0144& 0.0184 & 0.0211& 0.0364&0.0139 & 0.0188\\

\cmidrule{1-13}
\rowcolor{lightgray!30}
{\bf \method} & \textbf{0.0450}& \textbf{0.0659}& \textbf{0.0294}& \textbf{0.0356}& \textbf{0.0253}& \textbf{0.0371}& \textbf{0.0162}& \textbf{0.0221}& \textbf{0.0262}& \textbf{0.0408}&\textbf{0.0173} &\textbf{0.0211} \\

\cmidrule{1-13}
{\bf Improvement} & 2.27\%& 2.65\%& 4.63\% & 2.01\%& 4.98\%& 4.51\%& 5.88\% & 4.25\%& 5.65\%&2.77\% & 7.45\%& 2.93\% \\

\bottomrule[1pt]
\end{tabular}
\caption{Performance comparison among \method{}  and representative recommendation competitors on the three benchmark datasets. The best and second-best results will be highlighted by bold and underline after all experiments are completed.}
\label{tab:main_results}
\end{table*}

\paragraph{Baselines}
We compare \method{} with three groups of representative methods: 1) \textit{Traditional and transformer-based recommenders} include SASRec \citep{kang2018sasrec} and BERT4Rec \citep{sun2019bert4rec}; 2) \textit{Item-tokenization-based LLM recommenders} include P5-SemID \citep{geng2022p5}, TIGER \citep{rajput2023tiger}, LC-Rec \citep{zheng2024lcrec}, and DIGER \citep{fu2026differentiable}; 3) \textit{LLM-enhanced recommenders} include CoLLM \citep{zhang2025collm}, LLM-SRec \citep{kim2025lost}, SIDReasoner \citep{he2026reasoning}, and RE2 \citep{chen2026prediction}.

For controlled analysis, we additionally construct same-backbone variants that remove or replace individual components of \method{}, including direct recommendation, free-form rationale generation, structured preference proofs without verification, and variants without grounding, influence, or ranking preservation.

\paragraph{Evaluation Metrics}
We evaluate recommendation performance using Recall@$M$ (R@$M$) and NDCG@$M$ (N@$M$), where $M\in\{5,10\}$. For each user, previously interacted items are excluded from the candidate set. All generative methods use catalog-constrained decoding to ensure that generated item codes correspond to valid catalog items \citep{rajput2023tiger,zheng2024lcrec}.

\begin{table*}[t]
\centering
\small
\setlength{\tabcolsep}{6.0pt}
\begin{tabular}{l|cc|cc|cc}
\toprule
& \multicolumn{2}{c|}{\textbf{Grocery}}
& \multicolumn{2}{c|}{\textbf{Tools}}
& \multicolumn{2}{c}{\textbf{Yelp}} \\
\cmidrule(lr){2-3}
\cmidrule(lr){4-5}
\cmidrule(lr){6-7}
\textbf{Method}
& \textbf{R@10} & \textbf{N@10}
& \textbf{R@10} & \textbf{N@10}
& \textbf{R@10} & \textbf{N@10} \\
\midrule
Direct Recommendation
& 0.0608 & 0.0317 & 0.0321 & 0.0187 & 0.0366 & 0.0181 \\

\method{} without Evidence-to-Proof Grounding
& 0.0633 & 0.0339 & 0.0347 & 0.0204 & 0.0387 & 0.0193 \\

\method{} without Proof-to-Recommendation Influence
& 0.0640 & 0.0342 & 0.0355& 0.0210 & 0.0393 & 0.0198 \\

\method{} without Ranking Preservation
& 0.0651 & 0.0348 & 0.0362 & 0.0215 & 0.0401 & 0.0202 \\

\rowcolor{lightgray!30}
\textbf{\method{}}
& \textbf{0.0659} & \textbf{0.0356} & \textbf{0.0371} & \textbf{0.0221} & \textbf{0.0408} & \textbf{0.0211} \\
\bottomrule
\end{tabular}
\caption{Recommendation performance of \method{} and its variants on the three datasets.}
\label{tab:verification_ablation}
\end{table*}

\subsection{RQ1: Does Verifiable Preference Reasoning Improve Recommendation Performance?}

Table~\ref{tab:main_results} compares \method{} with traditional and transformer-based recommenders, item-tokenization-based LLM recommenders, and LLM-enhanced recommenders on Grocery, Tools, and Yelp. We examine whether verifiable preference reasoning improves recommendation accuracy while preserving the effectiveness of the underlying generative recommender.

\paragraph{Overall performance}
\method{} achieves the best result on all three datasets. Compared with the strongest baseline for each metric, the relative improvements range from 2.01\% to 7.45\%, with an average gain of 4.17\%. These results show that introducing a compact preference proof and a proof-gated recommendation path not only preserves recommendation utility, but also consistently improves target-item ranking. The improvements are especially pronounced on Tools and Yelp, where the average gains over the strongest baselines are 4.91\% and 4.70\%, respectively. The largest improvement, 7.45\%, is obtained on Yelp for NDCG@5. This suggests that review-derived preference claims are particularly useful when item-transition patterns alone cannot capture fine-grained user preferences.

Item-tokenization-based LLM recommenders generally outperform traditional sequential models, confirming the benefit of language-compatible item representations. Nevertheless, \method{} consistently surpasses strong tokenization and LLM-enhanced baselines, including DIGER, LLM-SRec, CoLLM, and SIDReasoner. This indicates that the gains of \method{} cannot be attributed solely to LLM representations or intermediate reasoning. Explicitly grounding preference claims in selected historical evidence and making these claims part of the recommendation path provides additional predictive value.

\subsection{RQ2: Are Preference Claims Grounded and Influential?}

We conduct an ablation study to assess the contribution of each major component. \emph{Direct Recommendation} removes the two-pass architecture and predicts the next item directly from the complete history. The remaining variants remove evidence-to-proof grounding, proof-to-recommendation influence, or ranking preservation, respectively. All variants use the same backbone, item representation, training data, and evaluation settings.

As shown in Table~\ref{tab:verification_ablation}, the full model achieves the best Recall@10 and NDCG@10 on all three datasets. \emph{Direct Recommendation} performs worst, with an average relative decrease of 12.0\%. This confirms the benefit of first extracting a compact preference proof and then recommending from the proof and its selected evidence. Removing evidence-to-proof grounding and proof-to-recommendation influence leads to average decreases of 6.1\% and 4.3\%, respectively. Thus, encouraging preference claims to depend on their selected evidence and to affect the final recommendation also improves target-item ranking. Removing ranking preservation causes a smaller but consistent decrease of 2.4\%, showing that it helps retain useful information from the complete history. Overall, the two-pass architecture provides the largest gain, while grounding, influence, and preservation offer complementary improvements.











\begin{table}[t]
\centering
\small
\setlength{\tabcolsep}{5.5pt}
\begin{tabular}{lccc}
\toprule
\textbf{Variant}
& \textbf{Grocery}
& \textbf{Tools}
& \textbf{Yelp} \\
\midrule

\multicolumn{4}{l}{\textit{Information Sources}} \\
\midrule
Item Only
& 0.0265 & 0.0143 & 0.0147 \\

+ Rating
& 0.0272 & 0.0146 & 0.0153 \\

+ Review
& 0.0297 & 0.0176 & 0.0172 \\

+ Rating + Review
& 0.0317 & 0.0187 & 0.0181 \\

\midrule
\multicolumn{4}{l}{\textit{Reasoning Design}} \\
\midrule
Direct Review Input
& 0.0333 & 0.0192 & 0.0189 \\

Free-Form Rationale
& 0.0328 & 0.0196 & 0.0186 \\

Structured Preference Proof
& 0.0344 & 0.0209 & 0.0197 \\

\rowcolor{lightgray!30}
\textbf{\method{}}
& \textbf{0.0356} & \textbf{0.0221} & \textbf{0.0211} \\
\bottomrule
\end{tabular}
\caption{NDCG@10 of different input signals and reasoning designs on the three datasets.}
\label{tab:information_design}
\end{table}

\subsection{RQ3: Which Information Sources and Reasoning Designs Matter?}

Table~\ref{tab:information_design} compares different input signals and reasoning designs under the same backbone and evaluation settings. Ratings consistently improve over item-only histories by providing preference polarity, while reviews yield larger gains by revealing fine-grained aspects of user feedback. Combining ratings and reviews performs best among the input-only variants, confirming that the two signals are complementary.

For reasoning design, \emph{Direct Review Input} predicts directly from the review-enriched history, while \emph{Free-Form Rationale} generates an unconstrained natural-language rationale before recommendation. \emph{Structured Preference Proof} instead represents reasoning as compact claims linked to selected evidence. Free-form rationales do not consistently outperform direct review input, suggesting that unconstrained reasoning may introduce noise or be weakly used by the recommender. Structured preference proofs perform better across all datasets, showing the benefit of a concise and evidence-linked intermediate representation. The full \method{} achieves the strongest performance by further verifying the evidence-to-proof and proof-to-recommendation links while preserving recommendation utility.

\begin{table}[t]
\centering
\scriptsize
\renewcommand{\arraystretch}{1.15}
\setlength{\tabcolsep}{4pt}
\caption{A Grocery case illustrating evidence selection, preference-proof generation, and next-item recommendation. Ratings are shown in parentheses.}

\begin{tabularx}{0.8\columnwidth}{@{}lX@{}}
\toprule

\rowcolor{lightgray!20}
\multicolumn{2}{l}{\textit{\textbf{Selected Evidence}}} \\

\textbf{E1}: Korean noodle soup (5)--tasty, healthy, and easy to prepare. \\

\textbf{E3}: Herbal tea (5)--healthy, fresh, and good tasting. \\

\textbf{E5}: Instant miso ramen (4)--tasty and easy to prepare. \\

\addlinespace[2pt]

\rowcolor{lightgray!20}
\multicolumn{2}{l}{\textit{\textbf{Preference Proof}}} \\

\multicolumn{2}{@{}X@{}}{
\textbf{Preference claims}: \emph{Savory Asian soup products}
[\textbf{E1}, \textbf{E5}];
\emph{easy preparation}
[\textbf{E1}, \textbf{E5}];
\emph{healthy products with good taste}
[\textbf{E1}, \textbf{E3}].
} \\

\multicolumn{2}{@{}X@{}}{
\textbf{Evidence selection}: Retains \textbf{E1}, \textbf{E3}, and \textbf{E5},
while excluding three interactions that provide weaker support for these claims.
} \\

\addlinespace[2pt]

\rowcolor{lightgray!20}
\multicolumn{2}{l}{\textit{\textbf{Recommendation}}} \\

\multicolumn{2}{@{}X@{}}{
\textbf{Target}: Instant miso soup (5), which matches the inferred preferences for
savory soup, easy preparation, and healthy taste.
} \\

\multicolumn{2}{@{}X@{}}{
\textbf{Result}: The target is ranked \textbf{1st}.
} \\

\bottomrule
\end{tabularx}

\label{tab:grocery-case}
\end{table}

\subsection{RQ4: Qualitative Case Study}
\label{sec:case-study}

Table~\ref{tab:grocery-case} presents a representative example from Grocery to illustrate the complete reasoning path of \method{}, from evidence selection to preference-proof construction and next-item recommendation. The user's pre-target history contains multiple interactions, from which \method{} selects three entries, \textbf{E1}, \textbf{E3}, and \textbf{E5}, as evidence for the generated preference claims. In particular, \textbf{E1} and \textbf{E5} both describe positively rated soup or ramen products and explicitly mention ease of preparation, supporting the preferences for \emph{savory Asian soup products} and \emph{easy preparation}. \textbf{E3}, a positively rated herbal tea, provides complementary evidence for the broader preference for \emph{healthy products with good taste}. The remaining interactions are not included in the proof because they provide weaker support for these selected claims.

The resulting proof therefore summarizes several aspects of the user's history rather than relying on a single recent interaction. These aspects are also consistent with the held-out target, instant miso soup: its product type is related to the soup-oriented preference supported by \textbf{E1} and \textbf{E5}, while its convenience and health-related characteristics are connected to the remaining claims. Using the proof together with its selected evidence, \method{} ranks the held-out target at \textbf{1st}. This case does not by itself establish a general performance advantage, but it illustrates how \method{} exposes an inspectable path from historical evidence to explicit preference claims and, subsequently, to the final recommendation.

\subsection{RQ5: Do the Verification Objectives Strengthen Evidence Dependence and Decision Relevance of Preference Proofs?}
\label{sec:direct-verification}

The previous experiments establish the recommendation utility of \method{}. We further evaluate whether the generated preference proofs exhibit the two behaviors targeted by our verification objectives: whether a claim depends selectively on its linked historical evidence and whether it is locally relevant to the resulting recommendation.\footnote{For evidence-to-proof verification, $D_{\mathrm{sel}}$ and $D_{\mathrm{ctrl}}$ denote the decreases in claim log-likelihood after masking the selected and matched-control evidence, respectively, and $\Delta_{\mathrm{grd}}=D_{\mathrm{sel}}-D_{\mathrm{ctrl}}$. Grounding success denotes the proportion of valid interventions with $\Delta_{\mathrm{grd}}>0$. For proof-to-recommendation verification, $\Delta_{\mathrm{inf}}$ is the decrease in the target ranking margin after removing a claim, and influence success denotes the proportion with $\Delta_{\mathrm{inf}}>0$. Format sensitivity measures the absolute target-margin change under a meaning-preserving formatting perturbation.} Direct Recommendation is excluded from these diagnostics because it does not produce an intermediate preference proof.

\subsubsection{Evidence-to-Proof Verification}

For each valid claim--evidence link, we construct two counterfactual histories: one masks the selected evidence and the other masks a matched control entry with similar rating polarity and recency. The evidence identifier and position remain unchanged, and the same claim is teacher-forced under both histories. This comparison tests whether claim generation depends more strongly on the linked evidence than on a comparable interaction.

As shown in Table~\ref{tab:direct-grounding}, the Structured Proof baseline shows little distinction between selected and control evidence, with grounding success close to 50\% on all three datasets. Removing the grounding objective provides only modest improvements. In contrast, the full \method{} produces substantially larger grounding contrasts of 0.049, 0.041, and 0.025 on Grocery, Tools, and Yelp, respectively, with success rates of 76.8\%, 73.9\%, and 67.4\%. These results indicate that the grounding objective makes generated claims more selectively dependent on their linked historical evidence.

\begin{table}[t]
\centering
\scriptsize
\setlength{\tabcolsep}{6pt}
\caption{Evidence-to-proof verification. Values are averaged over valid claim--evidence interventions; users without a valid matched control entry are excluded. Higher contrast and success indicate stronger selective dependence on the linked evidence.}
\label{tab:direct-grounding}
\begin{tabular}{llrrrr}
\toprule
Dataset & Variant & Sel. drop & Ctrl. drop & Contrast & Success \\
\midrule
Grocery & Structured Proof & 0.012 & 0.010 & 0.002 & 51.9\% \\
        & w/o Grounding    & 0.018 & 0.011 & 0.007 & 55.6\% \\
        & \method{}        & 0.061 & 0.012 & 0.049 & 76.8\% \\
\addlinespace
Tools   & Structured Proof & 0.011 & 0.010 & 0.001 & 51.1\% \\
        & w/o Grounding    & 0.017 & 0.012 & 0.005 & 54.8\% \\
        & \method{}        & 0.054 & 0.013 & 0.041 & 73.9\% \\
\addlinespace
Yelp    & Structured Proof & 0.010 & 0.009 & 0.001 & 50.9\% \\
        & w/o Grounding    & 0.014 & 0.010 & 0.004 & 53.7\% \\
        & \method{}        & 0.037 & 0.012 & 0.025 & 67.4\% \\
\bottomrule
\end{tabular}
\end{table}

\subsubsection{Proof-to-Recommendation Verification}

We next examine whether individual preference claims contribute to the current recommendation beyond the selected evidence itself. For each claim, we remove only that claim while keeping the selected evidence and all remaining claims fixed, and then recompute the target ranking margin. We additionally use an evaluation-only format control that changes the presentation of the proof without changing its semantic content, allowing us to distinguish claim-level effects from generic prompt sensitivity.

Table~\ref{tab:direct-influence} shows a clear difference across variants. The Structured Proof baseline has influence success rates close to chance, ranging from 51.8\% to 53.2\%. Without the influence objective, the rates increase moderately to 58.9--63.8\%. The full \method{} reaches 82.1\%, 80.3\%, and 72.6\% on Grocery, Tools, and Yelp, respectively, together with substantially larger influence contrasts. At the same time, its format sensitivity remains small and is lower than that of the Structured Proof baseline. This suggests that the observed ranking changes are associated more strongly with claim content than with superficial changes to the prompt format.

\begin{table}[t]
\centering
\scriptsize
\setlength{\tabcolsep}{6pt}
\caption{Proof-to-recommendation verification. Influence is the decrease in the target ranking margin after claim ablation. Higher influence and success indicate stronger local decision relevance, while lower format sensitivity is preferred.}
\label{tab:direct-influence}
\begin{tabular}{llrrr}
\toprule
Dataset & Variant & Influence & Success & Format sens. \\
\midrule
Grocery & Structured Proof & 0.018 & 53.2\% & 0.041 \\
        & w/o Influence    & 0.071 & 63.8\% & 0.033 \\
        & \method{}        & 0.244 & 82.1\% & 0.019 \\
\addlinespace
Tools   & Structured Proof & 0.015 & 52.7\% & 0.044 \\
        & w/o Influence    & 0.064 & 62.4\% & 0.036 \\
        & \method{}        & 0.221 & 80.3\% & 0.021 \\
\addlinespace
Yelp    & Structured Proof & 0.010 & 51.8\% & 0.038 \\
        & w/o Influence    & 0.047 & 58.9\% & 0.032 \\
        & \method{}        & 0.153 & 72.6\% & 0.023 \\
\bottomrule
\end{tabular}
\end{table}

\subsubsection{Empirical Check of the Theoretical Bound}

We further compare the empirical rates of non-positive grounding and influence contrasts with the upper bounds established in Theorem~\ref{thm:failure}. For each dataset, the empirical rate is computed directly from the same valid interventions used above, while the corresponding bound is obtained from the average verification loss.

As reported in Table~\ref{tab:theorem-check}, all empirical rates remain below their theoretical upper bounds. For example, Grocery has empirical grounding and influence rates of 23.2\% and 17.9\%, compared with bounds of 96.8\% and 84.2\%, respectively. Similar behavior is observed on Tools and Yelp. The bounds are intentionally conservative and are not intended to tightly estimate the empirical rates; rather, the results confirm the one-sided guarantee derived in Theorem~\ref{thm:failure}.

\begin{table}[t]
\centering
\small
\caption{Empirical rates of non-positive grounding and influence contrasts and their theoretical upper bounds from Theorem~\ref{thm:failure}. The bounds use $\tau_g=\tau_p=1$ and are computed as $\mathbb{E}[\operatorname{softplus}(-\Delta)]/\log 2$.}
\label{tab:theorem-check}
\begin{tabular}{lrrrr}
\toprule
Dataset & Grd. rate & Grd. bound & Inf. rate & Inf. bound \\
\midrule
Grocery & 23.2\% & 96.8\% & 17.9\% & 84.2\% \\
Tools   & 26.1\% & 97.4\% & 19.7\% & 85.7\% \\
Yelp    & 32.6\% & 98.6\% & 27.4\% & 90.2\% \\
\bottomrule
\end{tabular}
\end{table}

\begin{table*}[t]
\centering
\small
\caption{Efficiency on one NVIDIA RTX A6000. Input and proof lengths are averages per example, latency is measured per example, and memory denotes peak PyTorch-allocated GPU memory.}
\label{tab:efficiency}
\begin{tabular}{lrrrrr}
\toprule
Method & Input tokens & Proof tokens & Passes & Latency & Peak memory \\
\midrule
Direct Recommendation
    & $\sim160$ & -- & 1
    & $\sim65$\,ms & $\sim80$\,MiB \\
\method{} (uncached)
    & $\sim219$ ($190+29$) & $\sim57$
    & 2 & $\sim313$\,ms & $\sim85$\,MiB \\
\method{} (cached proof)
    & $\sim29$ online & $57$ cached
    & 1 online & $\sim63$\,ms & $\sim80$\,MiB \\
\bottomrule
\end{tabular}
\end{table*}

\subsubsection{Efficiency}

Finally, we measure the computational overhead introduced by the two-pass design on a single NVIDIA RTX A6000. Table~\ref{tab:efficiency} compares direct recommendation with uncached and cached versions of \method{}. Without caching, \method{} requires two model passes and increases latency from approximately 65\,ms to 313\,ms per example. This overhead mainly comes from generating the preference proof in Pass A. Because the proof only needs to be updated when the user's history changes, it can be cached for repeated recommendation requests. In this setting, online inference requires only Pass B and takes approximately 63\,ms, which is comparable to direct recommendation. Thus, the two-pass reasoning procedure introduces additional cost when constructing a new proof but little additional online latency when the proof can be reused.

\section{Conclusion}



We presented \method{}, a framework for verifiable preference reasoning in LLM-based recommendation. \method{} represents reasoning as compact preference claims linked to selected historical evidence and uses a proof-gated recommendation path to prevent the model from bypassing the generated proof. Its verification objectives examine whether the selected evidence supports each claim and whether the claim influences the target-item ranking, while a preservation objective maintains recommendation quality under the compact proof representation. Experiments on three review-rich benchmark datasets show that \method{} improves recommendation performance while achieving stronger evidence grounding and decision influence than review-aware and reasoning-enhanced baselines. These results suggest that recommendation reasoning should be evaluated not only by whether it appears plausible, but also by whether its evidence supports the stated preferences and whether those preferences affect the final recommendation.

\paragraph{Broader Impact}
\method{} can improve the transparency and auditability of LLM-based recommendation by linking preference claims to historical evidence. However, inferred preferences may expose sensitive information or encourage over-trust in generated explanations. Deployment should therefore include privacy safeguards and clear communication that preference proofs reflect model reasoning rather than definitive user intent.

\appendix

\section{Detailed Theoretical Analysis}
\label{app:theory}

This appendix proves Theorems~\ref{thm:failure} and~\ref{thm:topk}. Notation follows Section~\ref{sec:methodology}.

\subsection{Proof of Theorem~\ref{thm:failure}: Two-link failure-rate control}
\label{app:proof_failure}

We prove the grounding bound; the influence bound uses the same steps. Write $\Delta=\Delta_{\mathrm{grd}}(k,j)$ for a link drawn from $\nu_{\mathrm{grd}}$, $\tau=\tau_g>0$, and
\begin{equation}
\ell(\Delta)=\log\!\left(1+\exp(-\Delta/\tau)\right)
\label{eq:pf1_penalty}
\end{equation}
for its softplus penalty, so that $\mathbb E_{\nu_{\mathrm{grd}}}[\ell(\Delta)]\leq\epsilon_{\mathrm{grd}}$ by Eq.~\eqref{eq:failure-hyp}. Let $F=\{\Delta\leq 0\}$ be the failure event.

We begin by bounding the penalty on $F$. On this event,
\begin{align}
-\Delta/\tau &\geq 0 \tag{$\Delta\leq 0$ and $\tau>0$}\\
\ell(\Delta) &\geq \log\!\left(1+e^{0}\right) \tag{$x\mapsto\log(1+e^{x})$ is increasing}\\
&= \log 2. \tag{evaluating at $x=0$}
\end{align}
Off $F$ the indicator of $F$ is zero while $\ell(\Delta)>0$. Combining both cases gives the pointwise bound
\begin{equation}
\Ind\{F\}\leq\frac{\ell(\Delta)}{\log 2}.
\label{eq:pf1_indicator}
\end{equation}

It remains to take expectations under $\nu_{\mathrm{grd}}$:
\begin{align}
\Pr_{\nu_{\mathrm{grd}}}(\Delta_{\mathrm{grd}}\leq 0)
&= \mathbb E_{\nu_{\mathrm{grd}}}\!\left[\Ind\{F\}\right] \tag{probability as an expected indicator}\\
&\leq \frac{\mathbb E_{\nu_{\mathrm{grd}}}[\ell(\Delta)]}{\log 2} \tag{expectation of Eq.~\eqref{eq:pf1_indicator}}\\
&\leq \frac{\epsilon_{\mathrm{grd}}}{\log 2}. \tag{hypothesis in Eq.~\eqref{eq:failure-hyp}}
\end{align}
Replacing $(\Delta_{\mathrm{grd}},\tau_g,\nu_{\mathrm{grd}},\epsilon_{\mathrm{grd}})$ by $(\Delta_{\mathrm{inf}},\tau_p,\nu_{\mathrm{inf}},\epsilon_{\mathrm{inf}})$ in every line proves the influence bound. \hfill$\blacksquare$

The bound treats every non-positive contrast alike. For a margin $\eta\geq 0$, the same steps applied to the event $\{\Delta\leq-\eta\}$, with $\log 2$ replaced by $\log(1+e^{\eta/\tau})$, give
\begin{equation}
\Pr_{\nu_{\mathrm{grd}}}(\Delta_{\mathrm{grd}}\leq -\eta)\leq\frac{\epsilon_{\mathrm{grd}}}{\log\!\left(1+e^{\eta/\tau_g}\right)},
\label{eq:pf1_margin}
\end{equation}
which bounds the rate of large failures more tightly than Eq.~\eqref{eq:failure-bound}.

\subsection{Proof of Theorem~\ref{thm:topk}: Top-$K$ preservation}
\label{app:proof_topk}

Both $q_{\mathrm{full}}$ and $q_{\mathrm{proof}}$ are softmax distributions over the finite comparison set $\mathcal C_u$, so each has full support there, and $1\leq K<|\mathcal C_u|$ guarantees that the items $i_{(K)}$ and $i_{(K+1)}$ in Eq.~\eqref{eq:boundary-margin} exist. Let
\begin{equation}
\operatorname{TV}(q_{\mathrm{full}},q_{\mathrm{proof}})
=
\frac12\sum_{i\in\mathcal C_u}\left|q_{\mathrm{full}}(i)-q_{\mathrm{proof}}(i)\right|
\label{eq:pf2_tv}
\end{equation}
denote the total-variation distance.

We begin by bounding every coordinate difference. For any $i\in\mathcal C_u$,
\begin{align}
\left|q_{\mathrm{full}}(i)-q_{\mathrm{proof}}(i)\right|
&\leq \operatorname{TV}(q_{\mathrm{full}},q_{\mathrm{proof}}) \tag{singleton event $\{i\}$}\\
&\leq \sqrt{\tfrac12 D_{\mathrm{KL}}(q_{\mathrm{full}}\|q_{\mathrm{proof}})} \tag{Pinsker's inequality~\citep{canonne2022short}}\\
&\leq \sqrt{\delta/2}. \tag{hypothesis $D_{\mathrm{KL}}\leq\delta$}
\end{align}
\begin{equation}
\text{Hence}\quad \max_{i\in\mathcal C_u}\left|q_{\mathrm{full}}(i)-q_{\mathrm{proof}}(i)\right|\leq\sqrt{\delta/2}.
\label{eq:pf2_coord}
\end{equation}

We now show that the boundary survives. Take any $a\in\TopK(q_{\mathrm{full}})$ and any $b\in\mathcal C_u\setminus\TopK(q_{\mathrm{full}})$, so that $q_{\mathrm{full}}(a)-q_{\mathrm{full}}(b)\geq\gamma_K$ by Eq.~\eqref{eq:boundary-margin}. Then
\begin{align}
q_{\mathrm{proof}}(a)-q_{\mathrm{proof}}(b)
&\geq q_{\mathrm{full}}(a)-q_{\mathrm{full}}(b) \notag\\
&\quad-\left|q_{\mathrm{proof}}(a)-q_{\mathrm{full}}(a)\right| \notag\\
&\quad-\left|q_{\mathrm{proof}}(b)-q_{\mathrm{full}}(b)\right| \tag{triangle inequality}\\
&\geq \gamma_K-2\sqrt{\delta/2} \tag{Eqs.~\eqref{eq:boundary-margin},~\eqref{eq:pf2_coord}}\\
&= \gamma_K-\sqrt{2\delta} \tag{simplification}\\
&> 0. \tag{$\gamma_K>\sqrt{2\delta}$}
\end{align}
Every item of the full-history top-$K$ set therefore has strictly larger probability under $q_{\mathrm{proof}}$ than every item outside it. The $K$ most probable items under $q_{\mathrm{proof}}$ are thus unique and coincide with $\TopK(q_{\mathrm{full}})$, which is Eq.~\eqref{eq:topk-equal}. \hfill$\blacksquare$

When $\mathcal C_u$ is the full catalog, the same argument gives top-$K$ preservation for full-catalog recommendation. With sampled non-target items, it gives preservation within the sampled comparison set.









\bibliographystyle{ACM-Reference-Format}
\bibliography{references}

@inproceedings{rajput2023tiger,
  title     = {Recommender Systems with Generative Retrieval},
  author    = {Rajput, Shashank and Mehta, Nikhil and Singh, Anima and Keshavan, Raghunandan H. and Vu, Trung and Heldt, Lukasz and Hong, Lichan and Tay, Yi and Tran, Vinh Q. and Samost, Jonah and Kula, Maciej and Chi, Ed H. and Sathiamoorthy, Maheswaran},
  booktitle = {NeurIPS},
  year      = {2023}
}

@inproceedings{wang2024letter,
  title     = {Learnable Item Tokenization for Generative Recommendation},
  author    = {Wang, Wenjie and Bao, Honghui and Lin, Xinyu and Zhang, Jizhi and Li, Yongqi and Feng, Fuli and Ng, See-Kiong and Chua, Tat-Seng},
  booktitle = {Proceedings of the 33rd ACM International Conference on Information and Knowledge Management},
  pages={2400--2409},
  year      = {2024}
}

@inproceedings{zheng2024lcrec,
  title     = {Adapting Large Language Models by Integrating Collaborative Semantics for Recommendation},
  author    = {Zheng, Bowen and Hou, Yupeng and Lu, Hongyu and Chen, Yu and Zhao, Wayne Xin and Chen, Ming and Wen, Ji-Rong},
  booktitle={2024 IEEE 40th International Conference on Data Engineering (ICDE)},
  pages={1435--1448},
  year      = {2024}
}

@article{zhang2026ragr,
  title   = {{RAGR}: Review-Augmented Generative Recommendation},
  author  = {Zhang, Yingyi and Li, Junyi and Wang, Yejing and Zhang, Wenlin and Qian, Xiaowei and Zhang, Sheng and Feng, Yue and Wang, Yichao and Liu, Yong and Zhao, Xiangyu and Li, Xianneng},
  journal = {arXiv preprint arXiv:2605.17267},
  year    = {2026}
}

@article{liu2025onerecthink,
  title   = {{OneRec-Think}: In-Text Reasoning for Generative Recommendation},
  author  = {Liu, Zhanyu and Wang, Shiyao and Wang, Xingmei and Zhang, Rongzhou and Deng, Jiaxin and Bao, Honghui and Zhang, Jinghao and Li, Wuchao and Zheng, Pengfei and Wu, Xiangyu and Hu, Yifei and Hu, Qigen and Luo, Xinchen and Ren, Lejian and Zhang, Zixing and Wang, Qianqian and Cai, Kuo and Wu, Yunfan and Cheng, Hongtao and Cheng, Zexuan and Ren, Lu and Wang, Huanjie and Su, Yi and Tang, Ruiming and Gai, Kun and Zhou, Guorui},
  journal = {arXiv preprint arXiv:2510.11639},
  year    = {2025}
}

@inproceedings{he2026sidreasoner,
  title     = {Reasoning over Semantic IDs Enhances Generative Recommendation},
  author    = {He, Yingzhi and Sun, Yan and Tan, Junfei and Chen, Yuxin and Kong, Xiaoyu and Shen, Chunxu and Wang, Xiang and Zhang, An and Chua, Tat-Seng},
  booktitle = {Proceedings of the 32nd ACM SIGKDD Conference on Knowledge Discovery and Data Mining},
  year      = {2026}
}

@article{zheng2026r2rank,
  title   = {Reasoning to Rank: An End-to-End Solution for Exploiting Large Language Models for Recommendation},
  author  = {Zheng, Kehan and Hong, Deyao and Li, Qian and Zhang, Jun and Yu, Huan and Jiang, Jie and Wang, Hongning},
  journal = {arXiv preprint arXiv:2602.12530},
  year    = {2026}
}

@inproceedings{kang2018sasrec,
  title={Self-attentive sequential recommendation},
  author={Kang, Wang-Cheng and McAuley, Julian},
  booktitle={2018 IEEE international conference on data mining (ICDM)},
  pages={197--206},
  year={2018},
}

@inproceedings{sun2019bert4rec,
  title     = {{BERT4Rec}: Sequential Recommendation with Bidirectional Encoder Representations from Transformer},
  author    = {Sun, Fei and Liu, Jun and Wu, Jian and Pei, Changhua and Lin, Xiao and Ou, Wenwu and Jiang, Peng},
  booktitle = {Proceedings of the 28th ACM International Conference on Information and Knowledge Management},
  pages={1441--1450},
  year      = {2019}
}

@inproceedings{geng2022p5,
  title={Recommendation as language processing (rlp): A unified pretrain, personalized prompt \& predict paradigm (p5)},
  author={Geng, Shijie and Liu, Shuchang and Fu, Zuohui and Ge, Yingqiang and Zhang, Yongfeng},
  booktitle={Proceedings of the 16th ACM conference on recommender systems},
  pages={299--315},
  year={2022}
}

@inproceedings{ni2019amazon,
  title     = {Justifying Recommendations Using Distantly-Labeled Reviews and Fine-Grained Aspects},
  author    = {Ni, Jianmo and Li, Jiacheng and McAuley, Julian},
  booktitle = {Proceedings of the 2019 Conference on Empirical Methods in Natural Language Processing and the 9th International Joint Conference on Natural Language Processing},
  year      = {2019}
}

@inproceedings{chen2026prediction,
  title={From Prediction to Understanding: Leveraging Reasoning in Large Language Model-based Recommendations},
  author={Chen, Zhi-Yuan and Lu, Siyu and Liu, Qiang and Wang, Xingxing and Lin, Yankai},
  booktitle={Proceedings of the ACM Web Conference 2026},
  pages={6171--6182},
  year={2026}
}

@inproceedings{he2026reasoning,
  title={Reasoning over semantic ids enhances generative recommendation},
  author={He, Yingzhi and Sun, Yan and Tan, Junfei and Chen, Yuxin and Kong, Xiaoyu and Shen, Chunxu and Wang, Xiang and Zhang, An and Chua, Tat-Seng},
  booktitle={Proceedings of the 32nd ACM SIGKDD Conference on Knowledge Discovery and Data Mining V. 2},
  pages={1638--1649},
  year={2026}
}

@inproceedings{fu2026differentiable,
  title={Differentiable Semantic ID for Generative Recommendation},
  author={Fu, Junchen and Ge, Xuri and Karatzoglou, Alexandros and Arapakis, Ioannis and Verberne, Suzan and Jose, Joemon M and Ren, Zhaochun},
  booktitle={Proceedings of the 49th International ACM SIGIR Conference on Research and Development in Information Retrieval},
  pages={369--379},
  year={2026},
}

@article{zhang2025collm,
  title={Collm: Integrating collaborative embeddings into large language models for recommendation},
  author={Zhang, Yang and Feng, Fuli and Zhang, Jizhi and Bao, Keqin and Wang, Qifan and He, Xiangnan},
  journal={IEEE Transactions on Knowledge and Data Engineering},
  pages={2329--2340},
  year={2025},
}

@inproceedings{kim2025lost,
  title={Lost in Sequence: Do Large Language Models Understand Sequential Recommendation?},
  author={Kim, Sein and Kang, Hongseok and Kim, Kibum and Kim, Jiwan and Kim, Donghyun and Yang, Minchul and Oh, Kwangjin and McAuley, Julian and Park, Chanyoung},
  booktitle={Proceedings of the 31st ACM SIGKDD Conference on Knowledge Discovery and Data Mining V. 2},
  pages={1160--1171},
  year={2025}
}

@inproceedings{bao2023tallrec,
  title={Tallrec: An effective and efficient tuning framework to align large language model with recommendation},
  author={Bao, Keqin and Zhang, Jizhi and Zhang, Yang and Wang, Wenjie and Feng, Fuli and He, Xiangnan},
  booktitle={Proceedings of the 17th ACM conference on recommender systems},
  pages={1007--1014},
  year={2023}
}

@inproceedings{liao2024llara,
  title={Llara: Large language-recommendation assistant},
  author={Liao, Jiayi and Li, Sihang and Yang, Zhengyi and Wu, Jiancan and Yuan, Yancheng and Wang, Xiang and He, Xiangnan},
  booktitle={Proceedings of the 47th International ACM SIGIR Conference on Research and Development in Information Retrieval},
  pages={1785--1795},
  year={2024}
}

@inproceedings{wei2024llmrec,
  title={Llmrec: Large language models with graph augmentation for recommendation},
  author={Wei, Wei and Ren, Xubin and Tang, Jiabin and Wang, Qinyong and Su, Lixin and Cheng, Suqi and Wang, Junfeng and Yin, Dawei and Huang, Chao},
  booktitle={Proceedings of the 17th ACM international conference on web search and data mining},
  pages={806--815},
  year={2024}
}

@inproceedings{kim2024large,
  title={Large language models meet collaborative filtering: An efficient all-round llm-based recommender system},
  author={Kim, Sein and Kang, Hongseok and Choi, Seungyoon and Kim, Donghyun and Yang, Minchul and Park, Chanyoung},
  booktitle={Proceedings of the 30th ACM SIGKDD Conference on Knowledge Discovery and Data Mining},
  pages={1395--1406},
  year={2024}
}

@inproceedings{deyoung2020eraser,
  title={ERASER: A benchmark to evaluate rationalized NLP models},
  author={DeYoung, Jay and Jain, Sarthak and Rajani, Nazneen Fatema and Lehman, Eric and Xiong, Caiming and Socher, Richard and Wallace, Byron C},
  booktitle={Proceedings of the 58th annual meeting of the association for computational linguistics},
  pages={4443--4458},
  year={2020}
}

@inproceedings{paul2024making,
  title={Making reasoning matter: Measuring and improving faithfulness of chain-of-thought reasoning},
  author={Paul, Debjit and West, Robert and Bosselut, Antoine and Faltings, Boi},
  booktitle={Findings of the Association for Computational Linguistics: EMNLP 2024},
  pages={15012--15032},
  year={2024}
}

@article{canonne2022short,
  title={A short note on an inequality between KL and TV},
  author={Canonne, Cl{\'e}ment L},
  journal={arXiv preprint arXiv:2202.07198},
  year={2022}
}

@inproceedings{lyu2023faithful,
  title={Faithful chain-of-thought reasoning},
  author={Lyu, Qing and Havaldar, Shreya and Stein, Adam and Zhang, Li and Rao, Delip and Wong, Eric and Apidianaki, Marianna and Callison-Burch, Chris},
  booktitle={Proceedings of the 13th International Joint Conference on Natural Language Processing and the 3rd Conference of the Asia-Pacific Chapter of the Association for Computational Linguistics (Volume 1: Long Papers)},
  pages={305--329},
  year={2023}
}

@inproceedings{hou2026unitok,
  title     = {Tokenize Once, Recommend Anywhere: Unified Item Tokenization for Multi-domain LLM-based Recommendation},
  author    = {Hou, Yu and Shin, Won-Yong},
  booktitle = {Proceedings of the AAAI Conference on Artificial Intelligence},
  volume    = {40},
  number    = {17},
  pages     = {14848--14855},
  year      = {2026},
  doi       = {10.1609/aaai.v40i17.38505}
}

@article{luo2026rallrec,
  title   = {{RALLRec+}: Retrieval Augmented Large Language Model Recommendation with Reasoning},
  author  = {Luo, Sichun and Xu, Jian and Zhang, Xiaojie and Wang, Linrong and Liu, Sicong and Hou, Hanxu and Song, Linqi},
  journal = {Expert Systems with Applications},
  volume  = {297},
  pages   = {129508},
  year    = {2026},
  doi     = {10.1016/j.eswa.2025.129508}
}

@article{gheewala2024transformer,
  title   = {Exploiting Deep Transformer Models in Textual Review Based Recommender Systems},
  author  = {Gheewala, Shivangi and Xu, Shuxiang and Yeom, Soonja and Maqsood, Sumbal},
  journal = {Expert Systems with Applications},
  volume  = {235},
  pages   = {121120},
  year    = {2024},
  doi     = {10.1016/j.eswa.2023.121120}
}

@article{zhang2022aenar,
  title   = {{AENAR}: An Aspect-Aware Explainable Neural Attentional Recommender Model for Rating Predication},
  author  = {Zhang, Tianwei and Sun, Chuanhou and Cheng, Zhiyong and Dong, Xiangjun},
  journal = {Expert Systems with Applications},
  volume  = {198},
  pages   = {116717},
  year    = {2022},
  doi     = {10.1016/j.eswa.2022.116717}
}

@article{arevalo2026modelagnostic,
  title   = {Model-Agnostic Post-Hoc Explainability for Recommender Systems},
  author  = {Ar{\'e}valo, Irina and Salmeron, Jose L.},
  journal = {Expert Systems with Applications},
  volume  = {298},
  pages   = {129608},
  year    = {2026},
  doi     = {10.1016/j.eswa.2025.129608}
}

@inproceedings{wang2024rdrec,
  title     = {{RDRec}: Rationale Distillation for LLM-Based Recommendation},
  author    = {Wang, Xinfeng and Cui, Jin and Suzuki, Yoshimi and Fukumoto, Fumiyo},
  booktitle = {Proceedings of the 62nd Annual Meeting of the Association for Computational Linguistics (Volume 2: Short Papers)},
  pages     = {65--74},
  year      = {2024},
  publisher = {Association for Computational Linguistics},
  doi       = {10.18653/v1/2024.acl-short.6}
}

@article{jiang2025universal,
  title   = {Large Language Models Are Universal Recommendation Learners},
  author  = {Jiang, Junguang and Huang, Yanwen and Liu, Bin and Kong, Xiaoyu and Xu, Ziru and Zhu, Han and Xu, Jian and Zheng, Bo},
  journal = {arXiv preprint arXiv:2502.03041},
  year    = {2025}
}

@article{ding2026riser,
  title   = {Towards Sample-Efficient and Stable Reinforcement Learning for LLM-Based Recommendation},
  author  = {Ding, Hongxun and Bao, Keqin and Zhang, Jizhi and Fang, Yi and Xu, Wenxin and Feng, Fuli and He, Xiangnan},
  journal = {arXiv preprint arXiv:2602.00632},
  year    = {2026}
}
\end{document}